\documentclass[11pt,a4paper]{article}
\usepackage[hyperref]{acl2017}
\usepackage{times}
\usepackage{latexsym}
\usepackage{booktabs} 
\usepackage{bbold}
\usepackage{multirow}
\usepackage{graphicx}
\usepackage{url}
\usepackage{amsmath}
\usepackage{subfigure}
\usepackage[utf8]{inputenc}
\usepackage[english]{babel}
\usepackage{amsthm}
\newtheorem{theorem}{Theorem}

\newtheorem{lemma}[theorem]{Lemma}
\usepackage{placeins}
\usepackage{afterpage}

\aclfinalcopy 
\def\aclpaperid{506} 

\title{Cross-lingual Distillation for Text Classification}

\author{Ruochen Xu \\
  Carnegie Mellon University \\
  {\tt ruochenx@cs.cmu.edu} \\\And
  Yiming Yang \\
  Carnegie Mellon University \\
  {\tt yiming@cs.cmu.edu} \\}

\date{}

\begin{document}
\maketitle

\begin{abstract}
Cross-lingual text classification(CLTC) is the task of classifying documents written in different languages into the same taxonomy of categories. 
This paper presents a novel approach to CLTC that builds on model distillation, which adapts and extends a framework originally proposed for model compression. Using soft probabilistic predictions for the documents in a label-rich language as the (induced) supervisory labels in a parallel corpus of documents, we train classifiers successfully for new languages in which labeled training data are not available. An adversarial feature adaptation technique is also applied during the model training to reduce distribution mismatch. We conducted experiments on two benchmark CLTC datasets, treating English as the source language and German, French, Japan and Chinese as the unlabeled target languages. The proposed approach had the advantageous or comparable performance of the other state-of-art methods. \footnote{code available at \url{https://github.com/xrc10/cross-distill}}

\end{abstract}

\section{Introduction}
The availability of massive multilingual data on the Internet makes cross-lingual text classification (CLTC) increasingly important. The task is defined as to classify documents in different languages using the same taxonomy of predefined categories. 

CLTC systems build on supervised machine learning require a sufficiently amount of labeled training data for every domain of interest in each language. But in reality, labeled data are not evenly distributed among languages and across domains.  English, for example, is a label-rich language in the domains of news stories, Wikipedia pages and reviews of hotels, products, etc. But many other languages do not necessarily have such rich amounts of labeled data.  This leads to an open challenge in CLTC, i.e., how can we effectively leverage the trained classifiers in a label-rich \textit{source} language to help the classification of documents in other label-poor \textit{target} languages?

Existing methods in CLTC use either a bilingual dictionary or a parallel corpus to bridge language barriers and to translate classification models \cite{xu2016cross} or text data\cite{zhou2016cross}. There are limitations and challenges in using either type of resources. Dictionary-based methods often ignore the dependency of word meaning and its context, and cannot leverage domain-specific disambiguation when the dictionary on hand is a  general-purpose one.  
Parallel-corpus based methods, although more effective in deploying context (when combined with word embedding in particular), often have an issue of domain mismatch or distribution mismatch if the available source-language training data, the parallel corpus (human-aligned or machine-translation induced one) and the target documents of interest are not in exactly the same domain and genre\cite{duh2011machine}. How to solve such domain/distribution mismatch problems is an open question for research.
 
This paper proposes a new parallel-corpus based approach, focusing on the reduction of domain/distribution matches in CLTC.  We call this approach Cross-lingual Distillation with Feature Adaptation or CLDFA in short. It is inspired by the recent work in model compression \cite{hinton2015distilling} where a large ensemble model is transformed to a compact (small) model. The assumption of knowledge distillation for model compression is that the knowledge learned by the large model can be viewed as a mapping from input space to output (label) space. Then, by training with the soft labels predicted by the large model, the small model can capture most of the knowledge from the large model. Extending this key idea to CLTC, if we see parallel documents as different instantiations of the same semantic concepts in different languages, a target-language classifier should gain the knowledge from a well-trained source classifier by training with the target-language part of the parallel corpus and the soft labels made by the source classifier on the source language side. More specifically, we propose to

distillate knowledge from the source language to the target language in the following 2-step process: 
\begin{itemize}
\item Firstly, we train a source-language classifier with both labeled training documents and adapt it to the unlabeled documents from the source-language side of the parallel corpus. The adaptation enforces our classifier to extract features that are: 1) discriminative for the classification task and 2) invariant with regard to the distribution shift between training and parallel data.

\item Secondly, we use the trained source-language classifier to obtain the 
\textit{soft} labels for a parallel corpus, and the target-language part of the parallel corpus to train a target classifier, which yields a similar category distribution over target-language documents as that over source-language documents. We also use unlabeled testing documents in the target language to adapt the feature extractor in this training step.

\end{itemize}
Intuitively, the first step addresses the potential domain/distribution mismatch between the labeled data and the unlabeled data in the source language.  
The second step addresses the potential mismatch between the target-domain training data (in the parallel corpus) and the test data (not in the parallel corpus). The soft-label based training of target classifiers makes our approach unique among parallel-corpus based CLTC methods (Section \ref{sec:CLTC methods}. The feature adaptation step makes our framework particularly robust in addressing the distributional difference between in-domain documents and parallel corpus, which is important for the success of CLTC with low-resource languages.

The main contributions in this paper are the following:
\begin{itemize}
\item We propose a novel framework (CLDFA) for knowledge distillation in CLTC through a parallel corpus. It has the flexibility to be built on a large family of existing monolingual text classification methods and enables the use of a large amount of unlabeled data from both source and target language.
\item CLDFA has the same computational complexity as the plug-in text classification method and hence is very efficient and scalable with the proper choice of plug-in text classifier.
\item Our evaluation on benchmark datasets shows that our method had a better or at least comparable performance than that of other state-of-art CLTC methods.
\end{itemize}

\section{Related Work}
Related work can be outlined with respect to the representative work in CLTC and the recent progress in deep learning for knowledge distillation.

\subsection{CLTC Methods}
\label{sec:CLTC methods}
One branch of CLTC methods is to use lexical level mappings to transfer the knowledge from the source language to the target language. The work by Bel et al. \cite{bel2003cross} was the first effort to solve CLTC problem. They translated the target-language documents to source language using a bilingual dictionary. The classifier trained in the source language was then applied on those translated documents.  Similarly, Mihalcea et al. \cite{mihalcea2007learning} built cross-lingual classifier by translating subjectivity words and phrases in the source language into the target language. Shi et al. \cite{shi2010cross} also utilized a bilingual dictionary. 
Instead of translating the documents, they tried to translate the classification model from source language to target language. 

Prettenhofer and Stein. \cite{prettenhofer2010cross} also used the bilingual dictionary as a word translation oracle and built their CLTC system on structural correspondence learning, a theory for domain adaptation. A more recent work by \cite{xu2016cross} extended seminal bilingual dictionaries with unlabeled corpora in low-resource languages. Chen et al. \cite{chen2016adversarial} used bilingual word embedding to map documents in source and target language into the same semantic space, and adversarial training was applied to enforce the trained classifier to be language-invariant.

Some recent efforts in CLTC focus on the use of automatic machine translation (MT) technology. For example, Wan \cite{wan2009co} used machine translation systems to give each document a source-language and a target-language version, where one version is machine-translated from the another one. A co-training \cite{blum1998combining} algorithm was applied on two versions of both source and target documents to iterative train classifiers in both languages. MT-based CLTC also include the work on multi-view learning with different algorithms, such as majority voting\cite{amini2009learning}, matrix completion\cite{xiao2013novel} and multi-view co-regularization\cite{guo2012cross}.

Another branch of CLTC methods focuses on representation learning or the mapping of the induced representations in cross-language settings \cite{guo2012transductive, zhou2016cross, zhou2015learning,zhouattention, xiao2013novel, jagarlamudi2011improving, de2011knowledge, vinokourov2002inferring, platt2010translingual, littman1998automatic}. For example, Meng et al. \cite{meng2012cross} and Lu et al. \cite{lu2011joint} used a parallel corpus to learn word alignment probabilities in a pre-processing step. Some other work attempts to find a language-invariant (or interlingua) representation for words or documents in different languages using various techniques, such as
latent semantic indexing \cite{littman1998automatic}, kernel canonical correlation analysis \cite{vinokourov2002inferring}, matrix completion\cite{xiao2013novel}, principal component analysis \cite{platt2010translingual} and Bayesian graphical models \cite{de2011knowledge}.

\subsection{Knowledge Distillation}
The idea of distilling knowledge in a neural network was proposed by Hinton et al \cite{hinton2015distilling}, in which they introduced a student-teacher paradigm. Once the cumbersome teacher network was trained, the student network was trained according to soften predictions of the teacher network. In the field of computer vision, it has been empirically verified that student network trained by distillation performs better than the one trained with hard labels. \cite{hinton2015distilling,romero2014fitnets,ba2014deep}. Gupta et al.\cite{gupta2015cross} transfers supervision between images from different modalities(e.g. from RGB image to depth image). There are also some recent works applied distillation in the field of natural language. For example, Lili et al. \cite{mou2015distilling} distilled task specific knowledge from a set of high-dimensional embeddings to a low-dimensional space. Zhiting et al. used an iterative distillation method to transfer the structured information of logic rules into the weights of a neural network. Kim et al. \cite{kim2016sequence} applied knowledge distillation approaches in the field of machine translation to reduce the size of neural machine translation model. Our framework shares the same purpose of existing works that transfer knowledge between models of different properties, such as model complexity, modality, and structured logic. However, our transfer happens between models working on different languages. To the best of knowledge, this is the first work using knowledge distillation to bridge the language gap for NLP tasks.

\section{Preliminary}
\subsection{Task and Notation}
CLTC aims to use the training data in the source language to build a model applicable in the target language. In our setting, we have labeled data in source language $L_{src} = \{ x_i, y_i \}_{i=1}^{L}$, where $x_i$ is the labeled document in source language and $y_i$ is the label vector. We then have our test data in the target language, given by $T_{tgt} = \{ x'_i \}_{i=1}^{T}$. Our framework can also use unlabeled documents from both languages in transductive learning settings. We use $U_{src} = \{x_i\}_{i=1}^{M}$ to denote source-language unlabeled documents,$U_{tgt} = \{x'_i\}_{i=1}^{N}$ to denote target-language unlabeled documents, and $U_{parl}=\{ (x_i, x'_i) \}_{i=1}^P$ to denote a unlabeled bilingual parallel corpus where $x_i$ and $x'_i$ are paired document translations of each other.
We assume that the unlabeled parallel corpus does not overlap with the source-language training documents and the target-language test documents.

\subsection{Convolutional Neural Network (CNN) as a Plug-in Classifier}
We use a state-of-the-art CNN-based neural network classifier \cite{kim2014convolutional} as the plug-in classifier in our  framework. Instead of using a bag-of-words representation for each document, the CNN model concatenates the word embeddings (vertical vectors) of each input document into a $n \times k$ matrix, where $n$ is the length (number of word occurrences) of the document, and $k$ is the dimension of word embedding. Denoting by
$$ \mathbf{x_{1:n}} = \mathbf{x_1} \oplus \mathbf{x_2} \oplus ... \oplus \mathbf{x_n} $$ as the resulted matrix, with $\oplus$ the concatenation operator.
One-dimensional convolutional filter $\mathbf{w}\in R^{hk}$ with window size $h$ operates on every consecutive $h$ words, with non-linear function $f$ and bias $b$. For window of size $h$ started at index $i$, the feature after convolutional filter is given by:
$$c_i = f(\mathbf{w} \cdot \mathbf{x_{i:i+h-1}} + b)$$
A max-over-time pooling \cite{collobert2011natural} is applied on $c$ over all possible positions such that each filter extracts one feature. The model uses multiple filters with different window sizes. The concatenated outputs from filters consist the feature of each document. We can see the convolutional filters and pooling layers as feature extractor $\mathbf{f} = G_f(x, \theta_f)$, where $\theta_f$ contains parameters for embedding layer and convolutional layer. Theses features are then passed to a fully connected softmax layer to produce probability distributions over labels. We see the final fully connected softmax layer as a label classifier $G_y(\mathbf{f}, \theta_y)$ that takes the output $\mathbf{f}$ from the feature extractor. The final output of model is given by $G_y(G_f(x, \theta_f), \theta_y)$, which is jointly parameterized by $\{ \theta_f, \theta_y \}$

We want to emphasize that our choice of the plug-in classifier here is mainly for its simplicity and scalability to demonstrate our framework. There is a large family of neural classifiers for monolingual text classification that could be used in our framework as well, including other convolutional neural networks by \cite{johnson2014effective}, the recurrent neural networks by \cite{lai2015recurrent, zhang2016dependency,johnson2016supervised, sutskever2014sequence, dai2015semi}, the attention mechanism by \cite{yangateen}, the deep dense network by \cite{iyyer2015deep}, and more. 

\section{Proposed Framework}

Let us introduce two versions of our model for cross-language knowledge distillation, i.e., the vanilla version and the full version with feature adaptation. Both are supported by the proposed framework.  We denote the former by CLD-KCNN and the latter by CLDFA-KCNN.


\subsection{Vanilla Distillation}
\label{sec:vani_dist}
Without loss of generality, assume we are learning a multi-class classifier for the target language. We have $y \in {1,2,...,|\mathit{v}|}$ where $\mathit{v}$ is the set of all possible classes.  We assume the base classification network produces real number logits $q_j$ for each class. For example, for the case of CNN text classifier, the logits can be produced by a linear transformation which takes features extracted max-pooling layer and outputs a vector of size $|\mathit{v}|$. The logits are converted into probabilities of classes through the softmax layer, by normalizing each $q_j$ with all other logits.

\begin{equation}
p_j = \frac{\exp(q_j/T)}{\sum_{k=1}^{|\mathit{v}|} \exp(q_k/T) }
\end{equation}
where $T$ is a temperature and is normally set to 1. Using a higher value of $T$ produces a softer probability distribution over classes. 

The first step of our framework is to train the source-language classifier on labeled source documents $L_{src}$. We use standard temperature $T=1$ and cross-entropy loss as the objective to minimize. For each example and its label $(x_i, y_i)$ from the source training set, we have:

\begin{align}
\begin{split}
\label{eq:ce_loss}
& \mathcal{L}(\theta_{src}) = \\ 
& - \sum_{(x_i, y_i) \in L_{src}} \sum_{k=1}^{|\mathit{v}|} \mathbb{1}\{ y_i = k \} \log p(y=k|x_i; \theta_{src})
\end{split}
\end{align}

where $p(y=k|x; \theta_{src})$ is source model controlled by parameter $\theta_{src}$ and $\mathbb{1}\{ \cdot \}$ is the indicator function.

In the second step, the knowledge captured in $\theta_{src}$ is transferred to the distilled model in the target language by training it on the parallel corpus. The intuition is that paired documents in parallel corpus should have the same distribution of class predicted by the source model and target model. In the simplest version of our framework, for each source-language document in the parallel corpus, we predict a soft class distribution by source model with high temperature. Then we minimize the cross-entropy between soft distribution produced by source model and the soft distribution produced by target model on the paired documents in the target language. More formally, we optimize $\theta_{tgt}$ according to the following loss function for each document pair $(x_i, x'_i)$ in parallel corpus.
\begin{align}
\begin{split}
\label{eq:dist_loss}
& \mathcal{L}(\theta_{tgt}) = -\sum_{(x_i, x'_i) \in U_{parl}} \\
& \sum_{k=1}^{|\mathit{v}|} p(y=k|x_i; \theta_{src}) \log p(y=k|x'_i; \theta_{tgt})
\end{split}
\end{align}
During distillation, the same high temperature is used for training target model. After it has been trained, we set the temperature to 1 for testing.

We can show that under some assumptions, the two-step cross-lingual distillation is equivalent to distilling a target-language classifier in the target-language input space.

\begin{lemma}
\label{th:lemma}
Assume the parallel corpus $\{ x_i, x'_i \} \in U_{parl}$ is generated by $x_i' \sim p(X';\eta)$ and $x_i = t(x_i')$, where $\eta$ controls the marginal distribution of $x_i$ and $t$ is a differentiable translation function with integrable derivative. Let $f_{\theta_{src}}(t(x'))$ be the function that outputs soft labels of $p(y=k|t(x'); \theta_{src})$. The distillation given by equation \ref{eq:dist_loss} can be interpreted as distillation of a target language classifier $f_{\theta_{src}}(t(x'))$ on target language documents sampled from $p(X';\eta)$. 
\label{lemma1}
\end{lemma}

$f_{\theta_{src}}(t(x'))$ is the classifier that takes input of target documents, translates them into source documents through $t$ and makes prediction using the source classifier. If we further assume the testing documents have the same marginal distribution $P(X';\eta)$, then the distilled classifier should have similar generalization power as $f_{\theta_{src}}(t(x'))$.

\begin{theorem}
\label{th:theorem}
Let source training data $ x_i \in L_{src} $ has marginal distribution $ p(X;\lambda)$. Under the assumptions of lemma \ref{lemma1}, further assume $p(t(x');\lambda) = p(x';\eta)$, $p(y|t(x')) = p(y|x')$ and $t'(x') \approx C$, where $C$ is a constant. Then $f_{\theta_{src}}(t(x'))$ actually minimizes the expected loss in target language data $E_{x'\sim p(X;\eta), y\sim p(Y|x')} [L\big(y, f(t(x'))\big)]$.
\end{theorem}
\begin{proof}
By definition of equation \ref{eq:ce_loss}, $f_{\theta_{src}}(x)$ minimizes the expected loss $E_{x\sim p(X;\lambda), y\sim p(Y|x)}[L\big(y, f(x)\big)]$, where $L$ is cross-entropy loss in our case. Then we can write
\begin{align*}
& E_{x\sim p(X;\lambda), y\sim p(Y|x)}[L\big(y, f(x)\big)] \\
= & \int p(x;\lambda) \sum_{y} p(y|x) L\big(y, f(x)\big) dx \\
= & \int p(t(x');\lambda) \sum_{y} p(y|t(x')) L\big(y, f(t(x'))\big) t'(x') dx'\\
\approx & C \int p(x';\eta) \sum_{y} p(y|x') L\big(y, f(t(x'))\big) dx'\\
= & C E_{x'\sim p(X;\eta), y\sim p(Y|x')} [L\big(y, f(t(x'))\big)]
\end{align*}

\end{proof}

\subsection{Distillation with Adversarial Feature Adaptation}

\begin{figure}[ht]
\centering
\includegraphics[width=0.4\textwidth]{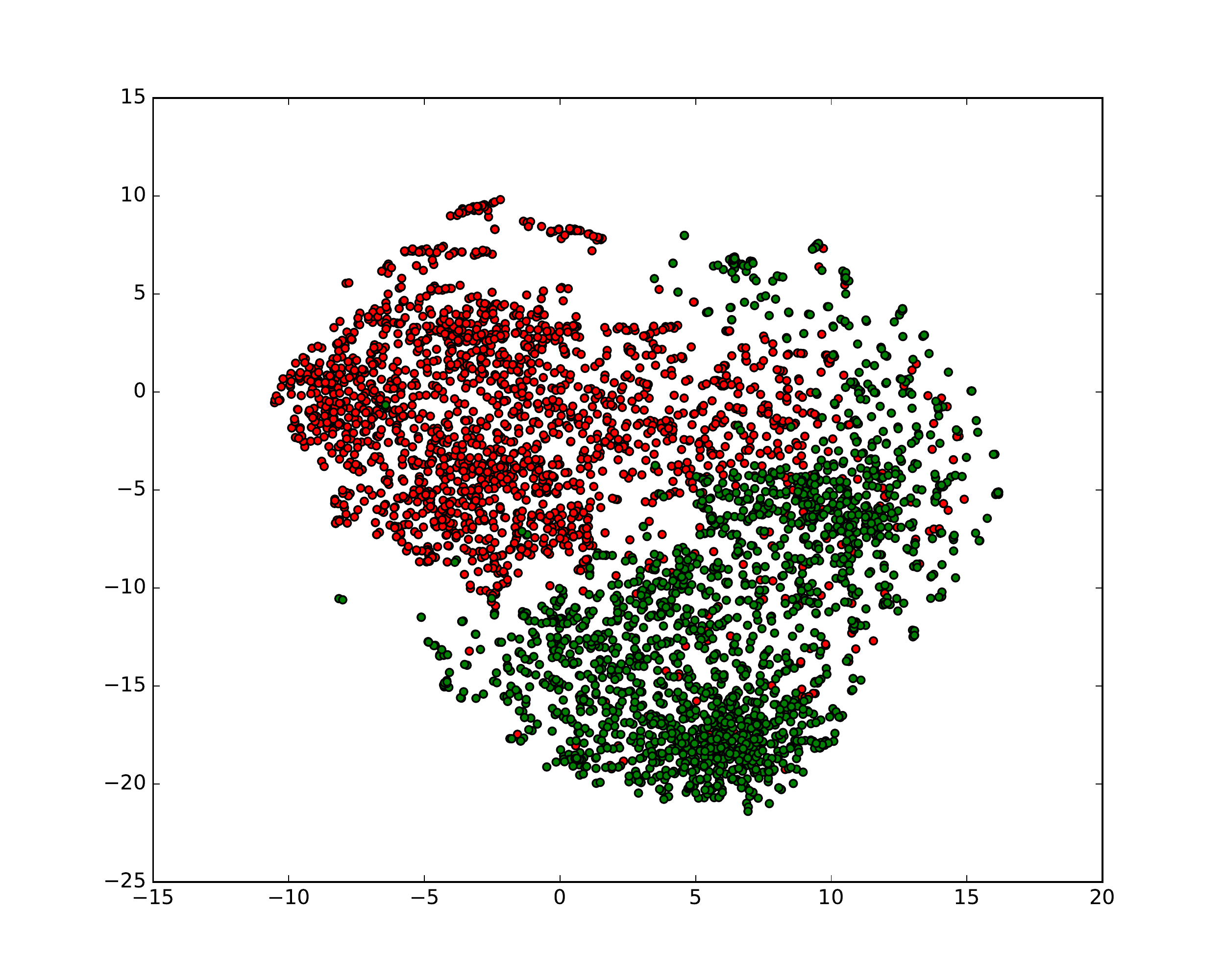}
\caption{Extracted features for source-language documents in the English-Chinese Yelp Hotel Review dataset. Red dots represent features of the documents in $L_{src}$ and green dots represent the features of documents in $U_{parl}$, which is a general-purpose parallel corpus.}
\label{fig:yelp_div}
\end{figure}

Although vanilla distillation is intuitive and simple, it cannot handle distribution mismatch issues. For example, the marginal feature distributions of source-language documents in $L_{src}$ and $U_{parl}$ could be different, so are the distributions of target-language documents in $U_{parl}$ and $T_{tgt}$. According to theorem 2, the vanilla distillation works for the best performance under unrealistic assumption: $p(t(x')|\lambda) = p(x'|\eta)$. To further illustrate our point, we trained a CNN classifier according to equation \ref{eq:ce_loss} and used the features extracted by $G_f$ to present the source-language documents in both $L_{src}$ and $U_{parl}$. Then we projected the high-dimensional features onto a 2-dimensional space via t-Distributed Stochastic Neighbor Embedding (t-SNE)\cite{maaten2008visualizing}.  This resulted the visualization of the project data in Figures \ref{fig:yelp_div} and \ref{fig:amazon_div}. 

\begin{figure*}[ht]
\centering
\subfigure[Germany:DVD]{
\includegraphics[width=.30\textwidth]{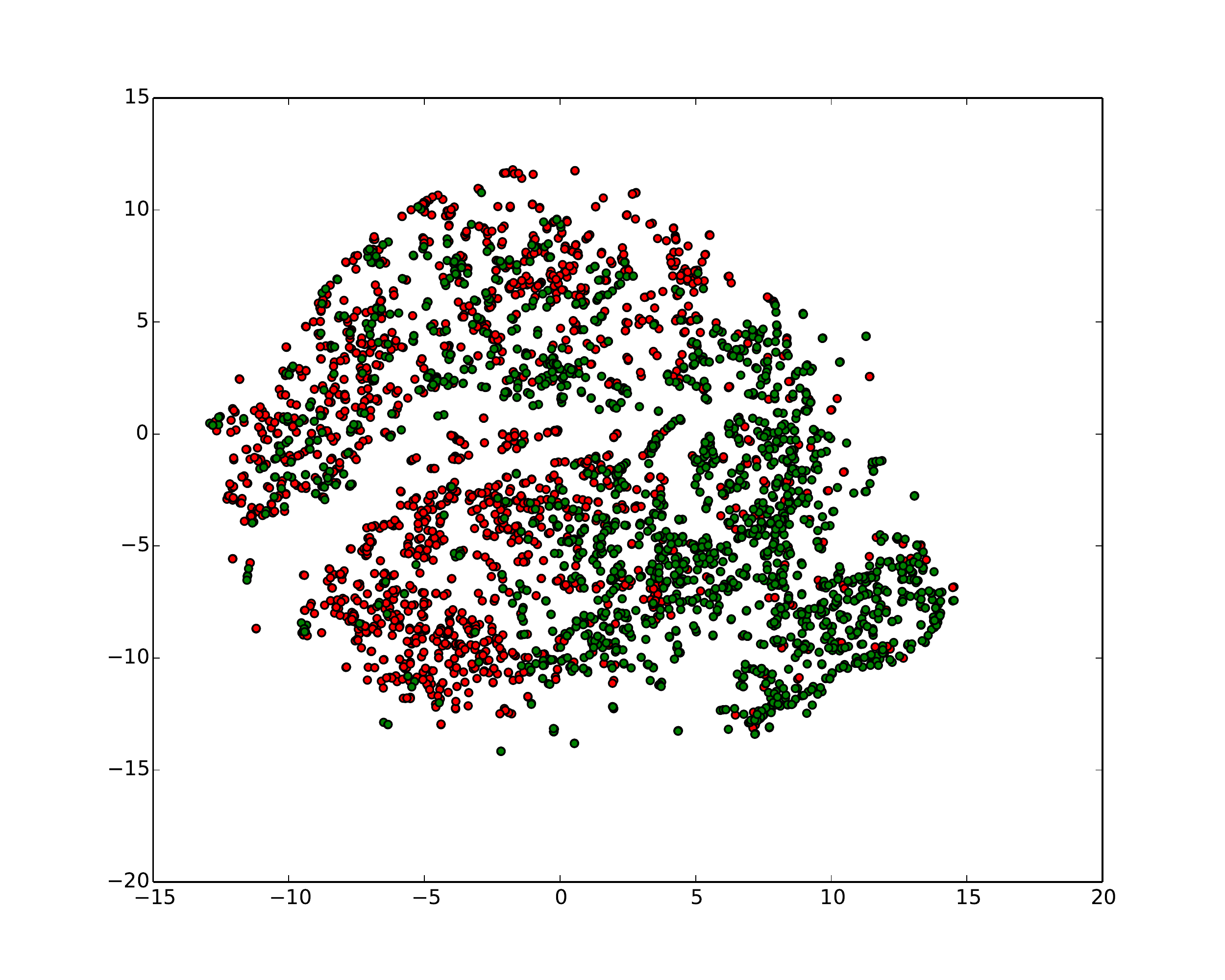}
}
\subfigure[French:Music]{
\includegraphics[width=.30\textwidth]{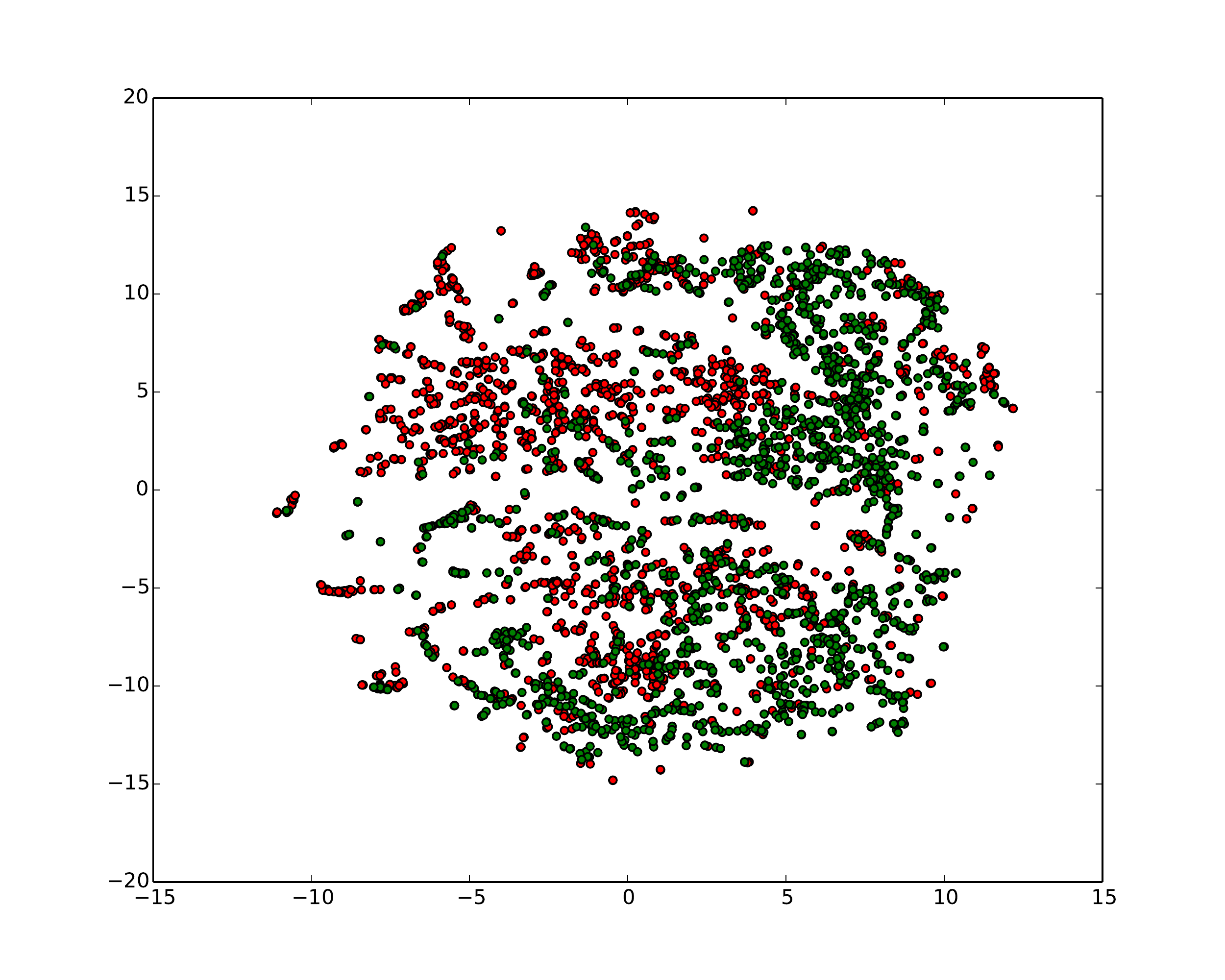}
}
\subfigure[Japanese:Book]{
\includegraphics[width=.30\textwidth]{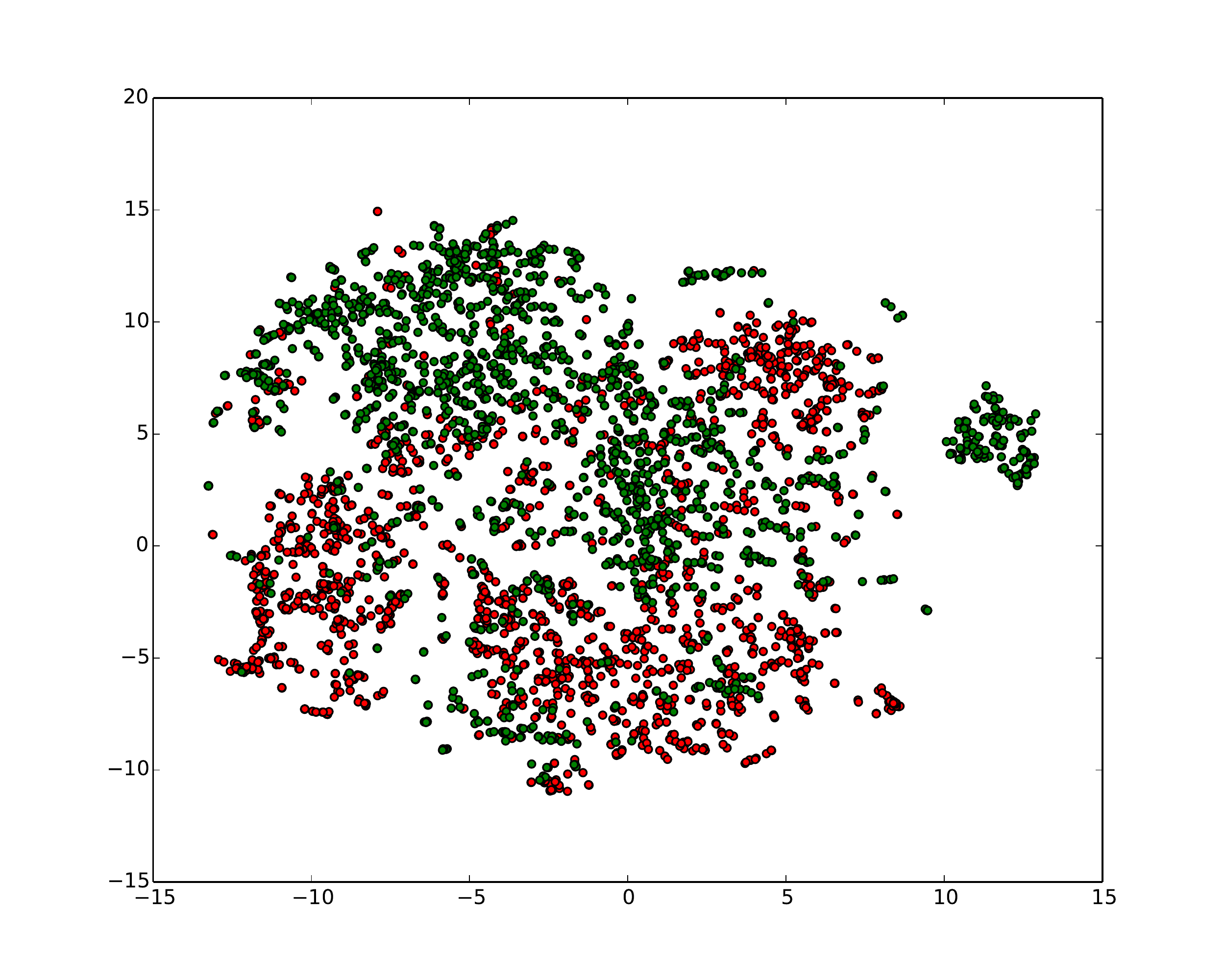}
}

\caption{Extracted features for the source-language documents in the Amazon Reviews dataset. Red dots represent the features of the labeled training documents in $L_{src}$, and green dots represent the features of the documents in $U_{parl}$, which are the machine-translated documents from a target language. Below each figure is the target language and the domain of review (Section \ref{sec:data}).}
\label{fig:amazon_div}
\end{figure*}

It is quite obvious in Figure \ref{fig:yelp_div} that the general-purpose parallel corpus has a very different feature distribution from that of the labeled source training set. Even for machine-translated parallel data from the same domain, as shown in figure \ref{fig:amazon_div}, there is still a non-negligible distribution shift from the source language to the target language for the extracted features. 
Our interpretation of this observation is that when the MT system (e.g. Google Translate) is a general-purpose one, it non-avoidably add translation ambiguities which would lead the distribution shift from the original domain.  To address the distribution divergence brought by either a general-purpose parallel corpus or an imperfect MT system, we seek to adapt the features extraction part of our neural classifier such that the feature distributions on both sides should be close as possible in the newly induced feature space. We adapt the adversarial training method by \cite{ganin2014unsupervised} to the cross-lingual settings in our problems. 

Given a set of training set of $L = \{ x_i, y_i \}_{i=1,...,N}$ and an unlabeled set $U = \{ x'_i \}_{i=1,...,M}$, our goal is to find a neural classifier $G_y(G_f(x, \theta_f), \theta_y)$, which has good discriminative performance on $L$ and also extracts features which have similar distributions on $L$ and $U$. One way to maximize the similarity of two distributions is to maximize the loss of a discriminative classifier whose job is to discriminate the two feature distributions. We denote this classifier by $G_d(\cdot, \theta_d)$, which is parameterized by $\theta_d$.

At training time, we seek $\theta_f$ to minimize the loss of $G_y$ and maximize the loss of $G_d$. Meanwhile, $\theta_y$ and $\theta_d$ are also optimized to minimize their corresponding loss. The overall optimization could be summarized as follows:

\begin{align*}
& E(\theta_f, \theta_y, \theta_d) = \sum_{x_i, y_i \in L} L_y(y_i, G_y(G_f(x_i, \theta_f), \theta_y)) \\
& - \alpha \sum_{x_i \in L} L_d(0, G_d(G_f(x_i, \theta_f), \theta_d)) \\ 
& - \alpha \sum_{x_j \in U} L_d(1, G_d(G_f(x_j, \theta_f), \theta_d))
\end{align*}

where $L_y$ is the loss function for true labels $y$, $L_d$ is loss function for binary labels indicating the source of data and $\alpha$ is the hyperparameter that controls the relative importance of two losses. We optimize $\theta_f, \theta_y$ for minimizing $E$ and optimize $\theta_d$ for maximizing $E$. We jointly optimize $\theta_f, \theta_y, \theta_d$ through the gradient reversal layer\cite{ganin2014unsupervised}.

We use this feature adaptation technique to firstly adapt the source-language classifier to the source-language documents of the parallel corpus. When training the target-language classifier by matching soft labels on the parallel corpus, we also adapt the classifier to the target testing documents. We use cross-entropy loss functions as $L_y$ and $L_d$ for both feature adaptation.

\section{Experiments and Discussions}
\subsection{Dataset}
\label{sec:data}
Our experiments used two benchmark datasets, as described below.

\subsubsection*{(1) Amazon Reviews}

\begin{table}[!htb]
\centering
\resizebox{0.35\textwidth}{!}{%
\begin{tabular}{@{}|lll|@{}}
\toprule
Language & Domain &  \# of Documents \\ \midrule
\multirow{3}{*}{English} & book & 50000 \\
 & DVD & 30000 \\
 & music & 25220 \\ \midrule
\multirow{3}{*}{German} & book & 165470 \\
 & DVD & 91516 \\
 & music & 60392 \\ \midrule
\multirow{3}{*}{French} & book & 32870 \\
 & DVD & 9358 \\
 & music & 15940 \\ \midrule
\multirow{3}{*}{Japanese} & book & 169780 \\
 & DVD & 68326 \\
 & music & 55892 \\ \bottomrule
\end{tabular}
}
\caption{Dataset Statistics for the Amazon reviews dataset}
\label{tab:stat_unlab_amz}
\end{table}

We used the multilingual multi-domain Amazon review dataset created by Prettenhofer and Stein \cite{prettenhofer2010cross}. The dataset contains Amazon reviews in three domains: book, DVD and music. Each domain has the reviews in four different languages: English, German, French and Japanese. We treated English as the source language and the rest three as the target languages, respectively. This gives us 9 tasks (the product of the 3 domains and the 3 target languages) in total. For each task, there are 1000 positive and 1000 negative reviews in English and the target language, respectively. \cite{prettenhofer2010cross} also provides 2000 parallel reviews per task, that were generated using Google Translate \footnote{translate.google.com}, and used by us for cross-language distillation. There are also several thousands of unlabeled reviews in each language. The statistics of unlabeled data is summarized in Table \ref{tab:stat_unlab_amz}. All the reviews are tokenized using standard regular expressions except for Japanese, for which we used a publicly available segmenter \footnote{https://pypi.python.org/pypi/tinysegmenter}.


\subsubsection*{(2) English-Chinese Yelp Hotel Reviews}
This dataset was firstly used for CLTC by \cite{chen2016adversarial}. The task is to make sentence-level sentiment classification with 5 labels(rating scale from 1 to 5), using English as the source language and Chinese as the target language. The labeled English data consists of balanced labels of 650k Yelp reviews from Zhang et al. \cite{zhang2015character}. The Chinese data includes 20k labeled Chinese hotel reviews and 1037k unlabeled ones from \cite{lin2015empirical}. Following the approach by \cite{chen2016adversarial}, we use 10k of labeled Chinese data as validation set and another 10k hotel reviews as held-out test data.
We 
a random sample of 500k parallel sentences from UM-courpus\cite{tian2014corpus}, which is a general-purpose corpus designed for machine translation.

\begin{table*}[!htb]
\centering
\resizebox{\textwidth}{!}{%
\begin{tabular}{llllllllll}
\hline
Target Language & Domain &\vline & PL-LSI & PL-KCCA & PL-OPCA & PL-MC & \vline & CLD-KCNN & CLDFA-KCNN \\ \hline
\multirow{3}{*}{German} & book &\vline & 77.59 & 79.14 & 74.72 & 79.22 & \vline & 82.54 & \textbf{83.95}* \\
 & DVD &\vline & 79.22 & 76.73 & 74.59 & 81.34 &\vline &  82.24 & \textbf{83.14}* \\
 & music &\vline & 73.81 & 79.18 & 74.45 & \textbf{79.39} &\vline &  74.65 & 79.02 \\ \hline
\multirow{3}{*}{French} & book &\vline & 79.56 & 77.56 & 76.55 &  81.92 & \vline & 81.6 & \textbf{83.37} \\
 & DVD &\vline & 77.82 & 78.19 & 70.54 & 81.97 &\vline &  82.41 & \textbf{82.56} \\
 & music &\vline & 75.39 & 78.24 & 73.69 & 79.3 &\vline &  83.01 & \textbf{83.31}* \\ \hline
\multirow{3}{*}{Janpanese} & book &\vline & 72.68 & 69.46 & 71.41 &  72.57 & \vline & 74.12 & \textbf{77.36}* \\
 & DVD &\vline & 72.55 & 74.79 & 71.84 & 76.6 &\vline &  79.67 & \textbf{80.52}* \\
 & music &\vline & 73.44 & 73.54 & 74.96 & 76.21 &\vline &  73.69 & \textbf{76.46} \\ \hline
\multicolumn{2}{l}{Averaged Accuracy} &\vline & 75.78 & 76.31 & 73.64  &  78.72 & \vline & 79.33 & \textbf{81.08}* \\ \hline
\end{tabular}
}
\caption{Accuracy scores of methods on the Amazon Reviews dataset: the best score in each row (a task) is highlighted in bold face. If the score of CLDFA-KCNN is statistically significantly better (in one-sample proportion tests) than the best among the baseline methods, it is marked using a star.}
\label{tab:amz_res}
\end{table*}

\begin{table}[!bht]
\centering
\begin{tabular}{ll}
\hline
Model & Accuracy \\ \hline
mSDA & 31.44\% \\
MT-LR & 34.01\% \\
MT-DAN & 39.66\% \\
ADAN & 41.04\% \\ \midrule
CLD-KCNN & 40.96\% \\ 
CLDFA-KCNN & \textbf{41.82\%} \\ \hline
\end{tabular}
\caption{Accuracy scores of methods on the English-Chinese Yelp Hotel Reviews dataset}
\label{tab:hotel_rev_res}
\end{table}

\subsection{Baselines}

We compare the proposed method with other state-of-the-art methods as outlined below. 

\subsubsection*{(1) Parallel-Corpus based CLTC Methods}
Methods in this category all use an unlabeled parallel corpus. Methods named \textbf{PL-LSI} \cite{littman1998automatic}, \textbf{PL-OPCA} \cite{platt2010translingual} and \textbf{PL-KCAA} \cite{vinokourov2002inferring} learn latent document representations in a shared low-dimensional space by performing the Latent Semantic Indexing (LSI), the Oriented Principal Component Analysis (OPCA) and a kernel (namely KCAA) for the parallel text. 
\textbf{PL-MC} \cite{xiao2013novel} recovers missing features via matrix Completion, and also uses $LSI$ to induce a latent space for parallel text. All these methods train a classifier in the shared feature space with labeled training data from both the source and target languages. 

\subsubsection*{(2) MT-based CLTC Methods}
The methods in this category all use an MT system to translate each test document in the target language to the source language in the testing phase. The prediction on each translated document is made by a source-language classifier, which can be a Logistic Regression model (\textbf{MT+LR}) \cite{chen2016adversarial} or a deep averaging network (\textbf{MT+DAN}) \cite{chen2016adversarial}.

\subsubsection*{(3) Adversarial Deep Averaging Network}
Similar to our approach, the adversarial Deep Averaging Network (\textbf{ADAN}) also exploits adversarial training for CLTC \cite{chen2016adversarial}.  However, it does not have the parallel-corpus based knowledge distillation part (which we do).  Instead, it uses averaged bilingual embeddings of words as its input and adapts the feature extractor to produce similar features in both languages. 

We also include the results of \textbf{mSDA} for the Yelp Hotel Reviews dataset. \textbf{mSDA} \cite{chen2012marginalized} is a domain adaptation method based on stacked denoising autoencoders, which has been proved to be effective in cross-domain sentiment classification evaluations. We show the results reported by \cite{chen2012marginalized}, where they used bilingual word embedding as input for mSDA. 

\subsection{Implementation Detail} 
We pre-trained both the source and target classifier with unlabeled data in each language. We ran word2vec\cite{mikolov2013efficient} \footnote{https://code.google.com/archive/p/word2vec/} on the tokenized unlabeled corpus. The learned word embeddings are used to initialize the word embedding look-up matrix, which maps input words to word embeddings and concatenates them into input matrix.

We fine-tuned the source-language classifier on the English training data with 5-fold cross-validation. For English-Chinese Yelp-hotel review dataset, the temperature $T$(Section \ref{sec:vani_dist}) in distillation is tuned on validation set in the target language. For Amazon review dataset, since there is no default validation set, we set temperature from low to high in $\{ 1, 3, 5, 10 \}$ and take the average among all predictions. 

\subsection{Main Results}

In tables \ref{tab:amz_res} and \ref{tab:hotel_rev_res} we compare the results of our methods (the vanilla version CLD-KCNN and the full version CLDFA-KCNN) with those of other methods based on the published results in the literature. The baseline methods are different in these two tables as they were previously evaluated (by their authors) on different benchmark datasets.  
Clearly, CLDFA-KCNN outperformed the other methods on all except one task in these two datasets, showing that knowledge distillation is successfully carried out in our approach.

Noticing that CLDFA-KCNN outperformed CLD-KCNN, showing the effectiveness of adversarial feature extraction in reducing the distribution mismatch between the parallel corpus and the train/test data in the target domain.
We should also point out that in Table \ref{tab:amz_res}, the four baseline methods (PL-LSI, PL-KCCA, PL-OPCA and PL-MC) were evaluated under the condition of using additional 100 labeled target documents for training, according to the author's report \cite{xiao2013novel}.  On the other hand, our methods (CLD-KCNN and CLDFA-KCNN) were evaluated under a tougher condition, i.e., not using any labeled data in the target domains. 

We also test our framework when a few training documents in the target language are available. A simple way to utilize the target-language supervision is to fit the target-language model with labeled target data after optimizing with our cross-lingual distillation framework. The performance of CLD-KCNN and CLDFA-KCNN trained with different sizes of labeled target-language data is shown in figure \ref{fig:tar_label}. We also compare the performance of training the same classifier using only the target-language labels(\textbf{Target Only} in figure \ref{fig:tar_label}). As we can see, our framework can efficiently utilize the extra supervision and improve the performance over the training using only the target-language labels. The margin is most significant when the size of the target-language label is relatively small.

\begin{figure}[ht]
\centering
\includegraphics[width=0.4\textwidth]{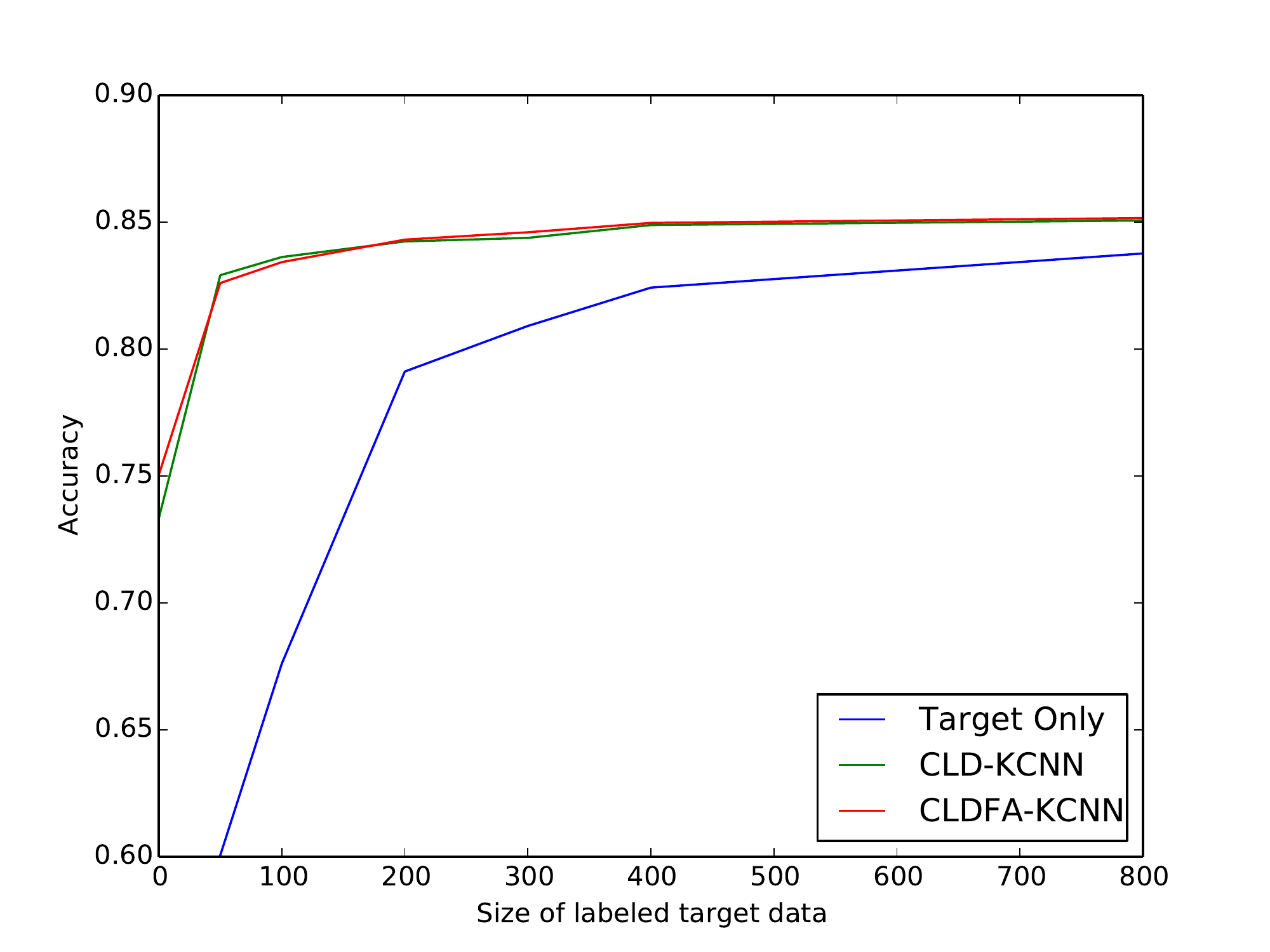}
\caption{Accuracy scores of methods using varying sizes of target-language labeled data on the Amazon review dataset. The target language is German and the domain is music. The parallel corpus has a fixed size of 1000 and the size of the labeled target-language documents is shown on the x-axis}
\label{fig:tar_label}
\end{figure}

\section{Conclusion}
This work introduces a novel framework for distillation of discriminative knowledge across languages, providing effective and efficient algorithmic solutions for addressing domain/distribution mismatch issues in CLTC. The excellent performance of our approach is evident in our evaluation on two CLTC benchmark datasets, compared to that of other state-of-the-art methods. 

\section*{Acknowledgement}
We thank the reviewers for their helpful comments. This work is supported in part by Defense Advanced Research Projects Agency Information Innovation Oce (I2O), the Low Resource Languages for Emergent Incidents (LORELEI) Program, Issued by DARPA/I2O under Contract No. HR0011-15-C-0114, by the National Science Foundation (NSF) under grant IIS-1546329.

\bibliography{acl2017}
\bibliographystyle{acl_natbib}

\end{document}